\documentclass[runningheads]{llncs}
\usepackage[T1]{fontenc}
\usepackage{times}
\usepackage{soul}
\usepackage{url}
\usepackage[hidelinks]{hyperref}
\usepackage[utf8]{inputenc}
\usepackage[small]{caption}
\usepackage{graphicx}
\usepackage{amsmath}
\usepackage{amsfonts}
\usepackage{mathtools}
\usepackage{mathrsfs}\usepackage{booktabs}
\usepackage{bm}
\usepackage{algorithm}
\usepackage{algorithmic}
\usepackage{framed}
\usepackage[switch]{lineno}
\mathtoolsset{showonlyrefs}
\usepackage{float}
\usepackage{lipsum}

\newcommand{\bmx}{\bm{x}}
\newcommand{\bmy}{\bm{y}}
\newcommand{\bmz}{\bm{z}}

\newcommand{\bmv}{\bm{v}}

\newcommand{\bmf}{\bm{f}}

\newcommand{\pr}{\mathrm{P}}

\newtheorem{mytheorem}[theorem]{Theorem}
\newtheorem{mylemma}[theorem]{Lemma}

\begin{document}
\title{A First Running Time Analysis of the
Strength Pareto Evolutionary Algorithm 2 (SPEA2)}
\titlerunning{Running Time Analysis of SPEA2}
%


\author{Shengjie Ren\inst{1}\and
Chao Bian\inst{1} \and
Miqing Li\inst{2} \and
Chao Qian\inst{1}
}

\authorrunning{Ren et al.}
\institute{National Key Laboratory for Novel Software Technology, Nanjing University, China \\
School of Artificial Intelligence, Nanjing University, China
\and
School of Computer Science, University of Birmingham, Birmingham B15 2TT, U.K.
\email{201300036@smail.nju.edu.cn,}
\email{bianc@lamda.nju.edu.cn}\\
\email{m.li.8@bham.ac.uk, }
\email{qianc@lamda.nju.edu.cn}}


\maketitle              
\begin{abstract}
Evolutionary algorithms (EAs) have emerged as a predominant approach for addressing multi-objective optimization problems. However, the theoretical foundation of multi-objective EAs (MOEAs), particularly the fundamental aspects like running time analysis, remains largely underexplored. Existing theoretical studies mainly focus on basic MOEAs, with little attention given to practical MOEAs. In this paper, we present a running time analysis of strength Pareto evolutionary algorithm 2 (SPEA2) for the first time. Specifically, we prove that the expected running time of SPEA2 for solving three commonly used multi-objective problems, i.e., $m$OneMinMax, $m$LeadingOnesTrailingZeroes, and $m$-OneJumpZeroJump, is $O(\mu n\cdot \min\{m\log n, n\})$, $O(\mu n^2)$, and $O(\mu n^k \cdot \min\{mn$,  $3^{m/2}\})$, respectively. Here $m$ denotes the number of objectives, and the population size $\mu$ is required to be at least $(2n/m+1)^{m/2}$, $(2n/m+1)^{m-1}$ and $(2n/m-2k+3)^{m/2}$, respectively. The proofs are accomplished through general theorems which are also applicable for analyzing the expected running time of other MOEAs on these problems, and thus can be helpful for future theoretical analysis of MOEAs.
\end{abstract}
\section{Introduction}

Multi-objective optimization requires optimizing several objectives at the same time, and it has been seen in many real-world scenarios. Since the objectives of a multi-objective optimization problem (MOP) are usually conflicting, there does not exist a single optimal solution, but instead a set of solutions which represent different optimal trade-offs between these objectives, called Pareto optimal solutions. The objective vectors of all the Pareto optimal solutions are called the Pareto front. The goal of multi-objective optimization is to find the Pareto front or a good approximation of it. 

Evolutionary algorithms (EAs)~\cite{back:96,eiben2015introduction} are a large class of randomized heuristic optimization algorithms inspired by natural evolution. They maintain a set of solutions (called a population), and iteratively improve it by generating new solutions and replacing inferior ones. 
The population-based nature makes EAs particularly effective in tackling MOPs, leading to their widespread application across various real-world domains~\cite{coello2007evolutionary,deb2001book,liang2024evolutionary,yang2023reducing,qian19el}.
Notably, there have been developed a multitude of well-established multi-objective EAs (MOEAs), including non-dominated sorting genetic algorithm II (NSGA-II)~\cite{deb2002fast}, strength Pareto evolutionary algorithm 2 (SPEA2)~\cite{zitzler2001spea2}, $\mathcal{S}$ metric selection evolutionary multi-objective optimization algorithm (SMS-EMOA)~\cite{beume2007sms}, and multi-objective evolutionary algorithm based on decomposition (MOEA/D)~\cite{zhang2007moea}.

In contrast to the wide applications of MOEAs, 
the theoretical foundation of MOEAs, especially the essential aspect, running time analysis, is still underdeveloped, which is mainly due to the sophisticated behaviors of EAs and the hardness of MOPs.
Early theoretical research primarily concentrates on analyzing the expected running time of GSEMO~\cite{giel2003expected} and SEMO~\cite{LaumannsTEC04} for solving a variety of multi-objective synthetic and combinatorial optimization problems~\cite{bian2018general,doerr2013lower,giel2003expected,giel2006effect,horoba2009analysis,neumann2007expected,neumann2010crossover,qian2011analysis}. 
Note that GSEMO is a simple MOEA which employs the bit-wise mutation operator to generate an offspring solution
in each iteration and keeps the non-dominated solutions generated-so-far in the population, and SEMO is a variant of GSEMO which employs one-bit mutation instead of bit-wise mutation. Furthermore, based on GSEMO and SEMO, the effectiveness of some parent selection strategies~\cite{osuna2020diversity,plateaus10,friedrich2011illustration,LaumannsTEC04}, mutation operator~\cite{doerr2021ojzj}, crossover operator~\cite{qian2011analysis}, and selection hyper-heuristics~\cite{qian2016selection}, has also been studied.

Recently, researchers have begun to examine practical MOEAs. The expected running time of $(\mu+1)$ SIBEA, a simple MOEA employing the hypervolume indicator for population update, has been analyzed across various synthetic problems~\cite{brockhoff2008analyzing,doerr2016runtime,nguyen2015sibea}, contributing to the theoretical understanding of indicator-based MOEAs. Subsequently, attention has turned to well-established algorithms in the evolutionary multi-objective optimization field. Huang \textit{et al.}~\cite{huang2021runtime} investigated MOEA/D, assessing the effectiveness of different decomposition methods by comparing running time for solving many-objective synthetic problems. Additionally, Zheng \textit{et al.}~\cite{zheng2023first} conducted the first analysis of the expected running time of NSGA-II by considering the bi-objective OneMinMax and LeadingOnesTrailingZeroes problems.
Bian \textit{et al.}~\cite{bian2023stochastic} analyzed the running time of SMS-EMOA~\cite{beume2007sms} for solving the bi-objective OneJumpZeroJump problem, and showed that a stochastic population update method can bring exponential acceleration.
Moreover, Wietheger and Doerr~\cite{wietheger23nsgaiii} demonstrated that NSGA-III~\cite{deb2014nsgaiii} exhibits superior performance over NSGA-II in solving the tri-objective problem $3$OneMinMax. In a very recent study, Lu \textit{et al.}~\cite{lu2024imoea} analyzed interactive MOEAs (iMOEAs), pinpointing scenarios where iMOEAs may work or fail.
Some other works about well-established MOEAs include~\cite{bian2022better,archive,cerf2023first,dang2023analysing,dang2023crossover,doerr2023first,doerr2023lower,doerr2023crossover,nsga3,multimodel,zheng2022better,zheng2023manyobj,zheng2024runtime}.

However, the running time analysis of SPEA2, one of the most popular MOEAs~\cite{zitzler2001spea2},
has not been touched.
SPEA2 employs \emph{enhanced fitness assignment} and \emph{density estimation} mechanisms to find well-spread Pareto optimal solutions. The former considers the number of individuals an individual dominates, and prefers those not dominated by others, hence promoting fitness in the population. The latter uses the $k$-th nearest neighbor method, which aids in maintaining the population's diversity and prevents premature convergence in the multi-objective optimization process.

In this paper, we analyze the expected running time of SPEA2 for solving three multi-objective problems, i.e., $m$OneMinMax, $m$LeadingOnesTrailingZeroes, and $m$-OneJumpZeroJump, which have been commonly used in the theory community of MOEAs~\cite{doerr2023first,LaumannsTEC04,nsga3,zheng2023first,zheng2023manyobj,zheng2024runtime}. Note that $m\ge 2$ denotes the number of objectives. 
Specifically, we prove that the expected number of fitness evaluations of SPEA2 for finding the Pareto front of the three problems is $O(\mu n\cdot \min\{m\log n, n\})$, $O(\mu n^2)$, and $O(\mu n^k \cdot \min\{mn$,  $3^{m/2}\})$, respectively, where the population size $\mu$ is required to be at least $(2n/m+1)^{m/2}$, $(2n/m+1)^{m-1}$ and $(2n/m-2k+3)^{m/2}$, correspondingly. 
The proofs are accomplished through general theorems which can also be applied for analyzing the expected running time
of other MOEAs.

The rest of this paper is organized as follows. Section~\ref{sec:Preliminary} introduces some preliminaries. The general theorems on the running time of MOEAs are provided in Section~\ref{sec:analysis_benchmark}, and these theorems are applied to SPEA2 in Section~\ref{sec:analysis_spea2} and other MOEAs in Section~\ref{sec:verification}. Finally, Section~\ref{conclusion} concludes the paper.

\section{Preliminaries}\label{sec:Preliminary}
  
In this section, we first introduce multi-objective optimization. Then, we introduce the analyzed algorithm, SPEA2. Finally, we present the $m$OneMinMax, $m$LeadingOnes-Trailingzeroes and $m$OneJumpZeroJump  problems considered in this paper.

\subsection{Multi-objective Optimization}

Multi-objective optimization seeks to optimize two or more objective functions concurrently, as presented in Definition~\ref{def_MO}. 
In this paper, we focus on maximization (though minimization can be similarly defined) and pseudo-Boolean functions whose solution space $\mathcal{X}=\{0,1\}^n$. 
Given that the objectives of a practical MOP typically conflict with each other, a canonical complete order within the solution space $\mathcal{X}$ does not exist. Instead, we employ the \emph{domination} relationship presented in Definition~\ref{def_Domination} to compare solutions.
A solution is deemed \emph{Pareto optimal} if no other solution in $\mathcal{X}$ dominates it, and the collection of objective vectors of all the Pareto optimal solutions is called the \emph{Pareto front}. 
The goal of multi-objective optimization is to identify the Pareto front or its good approximation.
\begin{definition}[Multi-objective Optimization]\label{def_MO}
	Given a feasible solution space $\mathcal{X}$ and objective functions $f_1,f_2,\ldots, f_m$, multi-objective optimization can be formulated as
	\[
	\max_{\bmx\in
		\mathcal{X}}\bmf(\bmx)=\max_{\bmx \in
		\mathcal{X}} \big(f_1(\bmx),f_2(\bmx),...,f_m(\bmx)\big).
	\]
\end{definition}
\begin{definition}[Domination]\label{def_Domination}
	Let $\bm f = (f_1,f_2,\ldots, f_m):\mathcal{X} \rightarrow \mathbb{R}^m$ be the objective vector. For two solutions $\bmx$ and $\bmy\in \mathcal{X}$:
	\begin{itemize}
		\item $\bmx$ \emph{weakly dominates} $\bmy$  (denoted as $\bmx \succeq \bmy$) if for any $1 \leq i \leq m, f_i(\bmx) \geq f_i(\bmy)$;
		\item $\bmx$ \emph{dominates} $\bmy$ (denoted as $\bmx\succ \bmy$) if $\bm{x} \succeq \bmy$ and $f_i(\bmx) > f_i(\bmy)$ for some $i$;
		\item  $\bmx$ and $\bmy$ are \emph{incomparable} if neither $\bmx\succeq \bmy$ nor $\bmy\succeq \bmx$.
	\end{itemize}
\end{definition}

\subsection{SPEA2}\label{def:spea2}

The SPEA2 algorithm~\cite{zitzler2001spea2}, as presented in Algorithm~\ref{alg:SPEA2}, is a well-established MOEA which employs a regular population $P$ and an archive $A$. It starts from an initial population of $\mu$ solutions and an empty archive $A$ (lines~1--2).  In each generation, it selects the non-dominated solutions in $P\cup A$ and adds them into an empty set $A'$ (line~4). 
If the size of $A'$ is larger than $\bar{\mu}$, SPEA2 uses the following truncation operator to remove the redundant solutions (lines~5--6). Let $\sigma_{\bmx}^k$ denote the distance of $\bmx$ to its $k$-th nearest neighbor in $A'$. We use $\bmx \le_d \bmy$ to denote that $\bmx$ has a smaller distance to its neighbour compared with $\bmy$. That is, 
\begin{equation}
\begin{aligned}
    \bmx \le_d \bmy \iff &(\forall 0 <k < |A'|: \sigma_{\bmx}^k = \sigma_{\bmy}^k) \vee \\ 
    & (\exists 0 < k < |A'|: [(\forall 0 < l < k: \sigma_{\bmx}^l = \sigma_{\bmy}^l) \wedge \sigma_{\bmx}^k < \sigma_{\bmy}^k ])\, \text{.}
\end{aligned}
\end{equation}
The truncation operator iteratively removes an individual $\bmx\in A'$ such that $\bmx \le_d \bmy$ for all $\bmy\in A'$, until $|A'| = \bar{\mu}$ (breaking a tie randomly). Note that once a solution is removed from $A'$, the $\sigma$ value will be updated.
If the size of $A'$ is smaller than $\bar{\mu}$, then the dominated solutions in $P\cup A$ are selected to fill the remaining slots according to their fitness (lines~7--8). The fitness of a solution is calculated as follows. 
First, let $S(\bmx) = |\{\bmy \in A\cup P \mid \bmx\succ \bmy\}|$ denote the strength of a solution $\bmx$, i.e., the number of solutions dominated by $\bmx$, and let $ R(\bmx) = \sum_{\bmy \in P \cup A, \bmy \succ \bmx }S(\bmy)$. 
We can see that a solution with smaller $R$ value is preferred, and  $R(\bmx) = 0$ implies that $\bmx$ is non-dominated.
Then, the fitness of a solution $\bmx$ is calculated as $F(\bmx)=R(x)+1/(\sigma_{\bmx}^k + 2)$.
After calculating the fitness of the dominated solutions in $P\cup A$, the solutions with the smallest fitness are selected into $A'$ such that the size of $A'$ equals to $\bar{\mu}$. 
After the modification of $A'$ finishes, the archive $A$ is set to $A'$ (line~10). Then, the population $P$ of size $\mu$ is formed by mutating the solutions selected from $A$ (lines~11--17).

\begin{algorithm}[t!]
	\caption{SPEA2}
	\label{alg:SPEA2}
	\textbf{Input}: objective functions $f_1,f_2...,f_m$, population size $\mu$, archive size $\bar{\mu}$  \\
	\textbf{Output}:  $\bar{\mu}$ solutions from $\{0,1\}^n$
	\begin{algorithmic}[1] 
		\STATE $P\leftarrow \mu$ solutions uniformly and randomly selected from $\{0,\! 1\}^{\!n}$ with replacement;
            \STATE $A = \emptyset$;
		\WHILE{criterion is not met}
            \STATE $A' \leftarrow $ non-dominated solutions in $ P \cup A$;
		\IF{$|A'| > \bar{\mu}$}
  		\STATE reduce $A'$ by means of the truncation operator
            \ELSIF{$|A'| < \bar{\mu}$}
  		\STATE fill $A'$ with dominated individuals in $P$ and $A$
		\ENDIF
            \STATE $A  \leftarrow A'$;
            \STATE let $P'=\emptyset$, $i=0$;
            \WHILE{$i<\mu$}
            \STATE select a solution from $A$ uniformly at random;
		\STATE generate $\bmx'$ by ﬂipping each bit of $\bmx$ independently with probability $1/n$;
		\STATE $P'= P'\cup \{\bmx'\}$, $i= i+1$
            \ENDWHILE
            \STATE $P\leftarrow P'$
		\ENDWHILE
		\RETURN $A$
	\end{algorithmic}
\end{algorithm}

\subsection{Benchmark Problems}

Now we introduce three multi-objective problems, i.e.,  $m$OneMinMax, $m$LeadingOnes-Trailingzeroes, and $m$OneJumpZeroJump, studied in this paper, where $m\ge 2$ is a positive even number and denotes the number of objectives.

The $m$OneMinMax problem presented in Definition~\ref{def:OMM} divides a solution into $m/2$ blocks, and in each block, the number of $0$-bits and the number of $1$-bits require to be maximized simultaneously. The Pareto front is $F^* = \{(i_1,2n/m - i_1, \cdots, i_{m/2},2n/m - i_{m/2}) \mid i_1,\cdots,i_{m/2} \in [0..2n/m]\}$, whose size is $(2n/m+1)^{m/2}$, and the Pareto optimal solution corresponding to $(i_1,2n/m - i_1, \cdots, i_{m/2},2n/m - i_{m/2})$ is the solution with $i_j$ 1-bits and $(2n/m - i_j)$ 0-bits in the $j$-th block. We can see that any solution $\bmx\in \{0,1\}^n$ is Pareto optimal.
\begin{definition}($m$OneMinMax~\cite{zheng2023manyobj})\label{def:OMM}
    Suppose $m$ is a positive even number, and $n$ is a multiple of $m/2$. The $m$OneMinMax problem of size $n$ is to find $n$-bits binary strings which maximize 
    \[\bmf(\bmx) = (f_1(\bmx), f_2(\bmx),\cdots,f_m(\bmx))\]
with
    \[f_k(\bmx) = \begin{cases}
	\sum_{i=1}^{2n/m}  \bmx_{i+n(k-1)/m}, & \text{if } k \text{ is odd},\\
	\sum_{i=1}^{2n/m}  (1-\bmx_{i+n(k-2)/m}), & \text{else}.
    \end{cases}\]

\end{definition}

The $m$LeadingOnesTrailingZeroes problem presented in Definition~\ref{def:LOTZ} also divides a solution into $m/2$ blocks, and in each block, the number of leading 1-bits and the number of trailing 0-bits require to be maximized simultaneously. The Pareto front is $F^* = \{(i_1,2n/m - i_1, \cdots, i_{m/2},2n/m - i_{m/2}) \mid i_1,\cdots,i_{m/2} \in [0..2n/m]\}$, whose size is $(2n/m+1)^{m/2}$, and the Pareto optimal solution corresponding to $(i_1,2n/m - i_1, \cdots, i_{m/2},2n/m - i_{m/2})$ is the solution with $i_j$ leading $1$-bits  and $(2n/m - i_j)$ trailing $0$-bits in the $j$-th block. 

\begin{definition}($m$LeadingOnesTrailingZeroes~\cite{LaumannsTEC04})\label{def:LOTZ}
    Suppose $m$ is a positive even number, and $n$ is a multiple of $m/2$. The $m$LeadingOnesTrailingZeroes problem of size $n$ is to find $n$-bits binary strings which maximize 
    \[\bmf(\bmx) = (f_1(\bmx), f_2(\bmx),\cdots,f_m(\bmx))\]
with
    \[f_k(\bmx) = \begin{cases}
	\sum_{i=1}^{2n/m} \Pi_{j=1}^i \bmx_{j+n(k-1)/m}, & \text{if } k \text{ is odd},\\
	\sum_{i=1}^{2n/m} \Pi_{j=i}^{2n/m} (1-\bmx_{j+n(k-2)/m}), & \text{else}.
    \end{cases}\]
\end{definition}

Before introducing the $m$OneJumpZeroJump problem, we first introduce the single-objective Jump problem, which aims at maximizing the number of 1-bits of a solution except for a valley around the solution with all 1-bits. Formally, the Jump problem of size $n$ aims to find an $n$-bits binary string which maximizes
\[\text{Jump}_{n,k}(\bmx) = \begin{cases}
	k+|\bmx|_1, & \text{if } |\bmx|_1 \le n-k \text{ or } |\bmx|_1 = n,\\
	n-|\bmx|_1, & \text{else},
    \end{cases}\]
where $2\le k\le n-1$ is a parameter and $|\bmx|_1$ denotes the number of 1-bits in $\bmx$.
The $m$OneJumpZeroJump problem presented in Definition~\ref{def:OJZJ} also divides a solution into $m/2$ blocks, and in each block, it tries to optimize a Jump problem as well as a counterpart of Jump problem with the roles of 1-bits and 0-bits exchanged. The Pareto front is $F^* = \{ (i_1, 2n/m +2k-i_1, \cdots, i_{m/2}, 2n/m + 2k -i_{m/2} ) \mid i_1,\cdots i_{m/2} \in [2k..2n/m] \cup \{ k, 2n/m +k\} \}$ whose size is $(2n/m -2k +3)^{m/2}$, and the Pareto optimal solution corresponding to $(i_1, 2n/m +2k-i_1, \cdots, i_{m/2}, 2n/m + 2k -i_{m/2} )$ is the solution with $(i_j-k)$ $1$-bits and $(2n/m-i_j + k)$ $0$-bits in the $j$-th block. 

\begin{definition}($m$OneJumpZeroJump~\cite{zheng2024runtime})\label{def:OJZJ}
     Suppose $m$ is a positive even number, and $n$ is a multiple of $m/2$. The $m$OneJumpZeroJump$_{n,k}$ problem of size $n$ is to find $n$-bits binary strings which maximize 
    
    \[\bmf(\bmx) = (f_1(\bmx), f_2(\bmx),\cdots,f_m(\bmx))\]
with
    \[f_i(\bmx) = \begin{cases}
	\text{Jump}_{2n/m,k}(\bmx_{[n(i-1)/m +1 .. n(i+1)/m]}), & \text{if } i \text{ is odd},\\
	\text{Jump}_{2n/m,k}(\bar{\bmx}_{[n(i-2)/m +1 .. ni/m]}), & \text{else},
    \end{cases}\]
where $\bar{\bmx} = (1-x_1, \cdots, 1-x_n)$.
\end{definition}

\section{General Theorems for Running Time Analysis of MOEAs} \label{sec:analysis_benchmark}

In this section, we present Theorems~\ref{thm:m-OMM}, \ref{thm:m-LOTZ} and~\ref{thm:m-OJZJ} that can be used to derive expected running time of general MOEAs for solving the $m$OneMinMax, $m$LeadingOnesTrailing-Zeroes and $m$OneJumpZeroJump problems, respectively. These results will be applied in Section~\ref{sec:analysis_spea2} to derive the running time bounds of SPEA2 and applied in Section~\ref{sec:verification} to derive the running time bounds of other MOEAs. 
In our analysis, we will use the general concept of a MOEA \emph{preserving the non-dominated set}, which is defined as follows:

\begin{definition}\label{def:preserve}
If a non-dominated solution $\bmx$ appears in the combined population of parent and offspring, then there will always be a solution $\bmy$ in the next generation such that $\bmf(\bmx) = \bmf(\bmy)$.
\end{definition}

For different MOEAs, the proportion between the parent and offspring populations varies. We assume that the parent population size is \(\mu\) and the offspring population size is \(c\mu\), where \(c \in [1/\mu, O(1)]\). For algorithms like SEMO, GSEMO, and SMS-EMOA, only one solution is generated in each iteration, thus \(c=1/\mu\), and for algorithms like NSGA-II and NSGA-III, \(c=1\).
Note that the running time of EAs is measured by the number of ﬁtness evaluations, which is the most time-consuming step in the evolutionary process. 


\subsection{On the $m$OneMinMax Problem} 

We prove in Theorem~\ref{thm:m-OMM} that the expected number of fitness evaluations for any MOEA solving $m$OneMinMax is $O(\mu n\cdot \min\{m\log n, n\})$, if the algorithm uses uniform selection and bit-wise mutation or one-bit mutation to generate offspring solutions, and can preserve the non-dominated set.
The proof idea is as follows. First, we show that the probability of not finding a specific Pareto front point in $O(m \mu n\log n)$ and $O(\mu n^2)$ fitness evaluations is at most $n^{-m}$ and $e^{-n}$, respectively. Then, we use the union bound to show that the probability of finding all Pareto front points in $O(\mu n\cdot \min\{m\log n, n\})$ fitness evaluations is $1-o(1)$. Note that this proof idea is inspired by Theorem~5.2 in~\cite{nsga3}, which analyzes NSGA-III solving $m$OneMinMax, and uses the union bound to derive the probability of finding the whole Pareto front within $T$ iterations, after deriving the probability of not finding a specific Pareto front point within $T$ iterations. Our proof differs by considering more general scenarios, such as varying proportions of parent and offspring populations and a larger number $m$ of objectives.

\begin{mytheorem}\label{thm:m-OMM}
    For any MOEA solving $m$OneMinMax, if the algorithm preserves the non-dominated set with a maximum population size of $\mu$, employs uniform selection to select parent solutions, and employs bit-wise mutation or one-bit mutation to generate offspring solutions, then the expected number of fitness evaluations for finding the Pareto front is $O(\mu n\cdot \min\{m\log n, n\})$.
\end{mytheorem}
%
%

\begin{proof}
For the $m$OneMinMax problem, all the solutions are Pareto optimal, and thus are non-dominated. Since the MOEA preserves the non-dominated set, any objective vector will be preserved in the population once it is found. 

First, we prove that for any objective vector $\bm{v}$, the probability of not finding it in $O(m \mu n \log n)$ fitness evaluations is at most $n^{-m}$. For any solution $\bmx$, we partition it into $m/2$ blocks, where the $i$-th ($i\in [1..m/2]$) block $B_i = [(i-1)(2n/m)+1 \ldots i(2n/m)]$. When there is a bit flip in block $\bmx_{B_i}$, one of $f_{2i-1}(\bmx)$ and $f_{2i}(\bmx)$ will increase by $1$, and the other one will decrease by $1$, such that $f_{2i-1}(\bmx)+ f_{2i}(\bmx) = 2n/m$ remains unchanged. Therefore, $\bmx$ needs to flip at least $\left\|\bmf(\bm{x}) - \bm{v} \right\|_{1} / 2$ bits to obtain $\bmv$. Let $d_{\bm{v}} = \min_{\bm{x} \in P} \left\|\bmf(\bm{x}) - \bm{v} \right\|_{1} / 2$ denote the minimum number of bits that require to be flipped to obtain $\bmv$ for all the solutions in the population $P$. Since $d_{\bm{v}} = \min_{\bm{x} \in P} \left\|\bmf(\bm{x}) - \bm{v} \right\|_{1} / 2 \le \left\|\bmf(\bm{x}) - \bm{v} \right\|_{1} / 2  \le \left\| 2 \bm{v} \right\|_{1} / 2  = n$, and $d_{\bmv} = 0$ when there exists an $\bmx \in P$ such that $\bmf(\bmx) = \bmv$, we have $0\le d_{\bmv} \le n$. Since all the objective vectors will be preserved, $d_{\bmv}$ cannot increase. Let the random variable $X_k, k \in [1..n]$ denote the number of generations with $d_{\bmv} = k$.  Let $X := \sum_{k=1}^{n} X_k$. When $d_{\bm{v}} = k$, the probability of selecting an individual $\bm{y}$ from the population such that $\left\|\bmf(\bm{y}) - \bm{v} \right\|_{1} / 2 = k$ is at least $1/\mu$. Since flipping any one of the $k$ bits corresponding to $d_{\bmv}$ will reduce $d_{\bmv}$ by $1$, the probability that $d_{\bm{v}}$ decreases is at least $(k/n) \cdot (1-1/n)^{n-1} \ge k/(en)$ when using bit-wise mutation and  at least $k/n$ when using one-bit mutation. In the following, we consider bit-wise mutation, while the analysis for one-bit mutation holds analogously. In each generation, the probability that $d_{\bm{v}}$ decreases is at least:
\begin{equation}\label{eq:OMM1}
\begin{aligned}
1-\Big( 1 - \frac{k}{e \mu n} \Big)^{c \mu} \ge 1 - e^{-\frac{ck}{en}} \ge \frac{ck}{ck + en} \ge \frac{ck}{(c+e)n},
\end{aligned}
\end{equation}
where $c\mu$ is the size of the offspring population, the second inequality holds by $1+a \le e^{a}$ for any $a\in \mathbb{R}$, and the third inequality holds by $k \le n $. Let $p_k := ck/(c+e)n$. 
Then, $X_k$ stochastically dominates a geometric random variable $Y_k$ with success probability $p_k$. That is for any $\lambda \ge 0$, we have $\pr(X_k \ge \lambda) \le \pr(Y_k \ge \lambda)$. Although $X_1,\cdots,X_n$ are not independent, it can be guaranteed that the success probability of $X_k$ is at least $p_k$. Therefore, for $Y:= \sum_{k=1}^{n}Y_k$ where $Y_1, \cdots Y_n$ are independent, $X$ stochastically dominates $Y$. By Theorem 16 in ~\cite{doerr2019analyzing} we have:
\begin{equation}\label{eq:OMM2}
\pr\Big(X \ge \frac{e+c}{c} (1+ m) n \log n\Big) \le \pr\Big(Y \ge \frac{e+c}{c} (1+ m) n \log n\Big) \le n^{-m}.
\end{equation}

\noindent In each generation, the algorithm produces $c\mu$ offspring solutions, thus the probability that the population contains no solution $\bmx$ with $\bmf(\bmx) = \bmv$ after $c \mu \cdot (e+c)(1+m)n(\log n)/c = (e+c)(1+m)\mu n\log n$ fitness evaluations is at most $n^{-m}$. 

Next, we prove that for any objective vector $\bm{v}$, the probability of not finding it in $O(\mu n^2 )$ fitness evaluations is at most $e^{-n}$. Let $p_k' := ck/(ck + en)$. By Eq.~\eqref{eq:OMM1}, $X_k$ stochastically dominates a geometric random variable $Z_k$ with success probability $p_k'$. Applying Theorem 1 in \cite{witt2014fitness} to the variable $Z := \sum_{k=1}^{n}Z_k$, where $Z_1 ,\cdots, Z_n$ are independent, we can obtain the following equation that is similar to Eq.~\eqref{eq:OMM2}:
\begin{equation}\label{eq:OMM3}
\begin{aligned}
&\pr\Big(X \ge \frac{5e n^2}{c}\Big) \le \pr\Big(Z \ge \frac{5e n^2}{c} \Big)  \\
&= \pr( Z \ge E[Z] + \delta) \le e^{-\frac{1}{4}\min\{\frac{\delta^2}{s},\frac{c \delta}{c + en}\}} \le e^{-n} \,\text{,}
\end{aligned}
\end{equation}

\noindent where $\delta = 5en^2/c - en H_n /c - n$, $E[Z] = en H_n /c + n$ and $s = \sum_{k=1}^{n} 1/p_k'^2 = n+ 2en H_n/c + (e\pi n)^2/6c$, with $H_n$ representing the $n$-th harmonic number, i.e. $\sum_{i=1}^n (1/i)$. Thus, the probability that the population contains no solution $\bmx$ with $\bmf(\bmx) = \bmv$ after $c \mu \cdot 5en^2/c = 5e\mu n^2$ fitness evaluations is at most $e^{-n}$. 

Finally, we consider finding the whole Pareto front. Recall that for the $m$OneMinMax problem, the size of the Pareto front is $(2n/m+1)^{m/2}$. Then, by applying the union bound, the probability of finding the whole Pareto front in $(e+c)(1+m)\mu n\log n$ number of fitness evaluations is at least $1-(2n/m +1)^{m/2} \cdot n^{-m} \ge 1 - (n+1)^{m/2} \cdot n^{-m} = 1 - O(n^{-m/2}) = 1-o(1)$, and the probability of finding the whole Pareto front in $5e\mu n^2$ number of fitness evaluations is at least $1-(2n/m +1)^{m/2} \cdot e^{-n} \ge 1 - 3^{n/2}\cdot e^{-n} = 1-o(1)$. The inequality holds because the derivative of $(2n/m +1)^{m/2}$ with respect to $m$ is always positive when $m \le n$, implying that when $m=n$, $(2n/m + 1)^{m/2}$ takes its maximum value, which is $3^{n/2}$. 
Combining the two parts, the whole Pareto front can be found in $\min\{(e+c)(1+m)\mu n\log n,5e\mu n^2 \}$ number of fitness evaluations with probability $1 - o(1)$. Let each $\min\{(e+c)(1+m)\mu n\log n,5e\mu n^2 \}$ fitness evaluations to find Pareto front be an independent trial with success probability of $1-o(1)$. Then, the expected number of fitness evaluations to find the Pareto front is at most  $(1-o(1))^{-1} \min\{(e+c)(1+m)\mu n\log n,5e\mu n^2 \} = O(\mu n\cdot \min\{m\log n, n\})$.
\qed
\end{proof}

Hence, if the MOEA preserves the non-dominated set, the expected number of fitness evaluations for solving $m$OneMinMax is $O(\mu n \cdot \min\{m \log n, n\})$. Specifically, when $m$ is a constant, $O(\min\{m \log n, n\}) = O(m \log n)$, leading to an expected number of fitness evaluations of $O(m \mu n \log n)$. Conversely, if the number $m$ of objectives is large, e.g., $m=n/4$,  the expected number of fitness evaluations is $O(\mu n^2)$.

\subsection{On the $m$LeadingOnesTrailingZeroes Problem}

We prove in Theorem~\ref{thm:m-LOTZ} that expected the number of fitness evaluations for any MOEA solving $m$LeadingOnesTrailingZeroes is $O(\mu n^2)$. The proof idea is to divide the optimization procedure into two phases, where the first phase aims at finding a point in the Pareto front. Let $w_i(\bmx) := f_{2i-1}(\bmx) + f_{2i}(\bmx)$ denote the sum of leading ones and trailing zeroes of block $\bmx_{B_i}$. 
Let $W_{\max}$ denote $\max_{\bmx \in P} \sum_{i=1}^{m/2} w_i(\bmx)$. We will prove that the expected number of fitness evaluations for $W_{\max}$ to reach $n$ (i.e., a Pareto front point is found) is at most $O(\mu n^2)$. The second phase aims at finding the whole Pareto front extended from the Pareto front point found in the first phase. Similar to the analysis on $m$OneMinMax, we prove that the probability of not finding a specific point in the Pareto front within $6e \mu n^2$ number of fitness evaluations is at most $e^{-n}$ by applying Theorem~1 in \cite{witt2014fitness}; by the union bound, the probability of finding the whole Pareto front within $O(\mu n^2)$ number of fitness evaluations is at least $1-(2n/m+1)^{m/2}\cdot e^{-n} = 1-o(1)$. Since the proof is similar to that of Theorem~4.3 in \cite{nsga3}, and also resembles that of Theorem~\ref{thm:m-OMM}, we omit the detailed proof here.

\begin{mytheorem}\label{thm:m-LOTZ}
    For any MOEA solving $m$LeadingOnesTrailingZeroes, if the algorithm preserves the non-dominated set with a maximum population size of $\mu$, employs uniform selection to select parent solutions, and employs bit-wise mutation or one-bit mutation to generate offspring solutions, then the expected number of fitness evaluations for finding the Pareto front is $O(\mu n^2)$.
\end{mytheorem}

Theorem~\ref{thm:m-LOTZ} shows that for any MOEA preserving the non-dominated set, the expected number of fitness evaluations for solving the $m$LeadingOnesTrailingZeroes problem is $O(\mu n^2)$. Notably, although this result seems unrelated to the number $m$ of objectives, the Pareto front size of $m$LeadingOnesTrailingZeroes (i.e., $(2n/m+1)^{m/2}$) amplifies with increasing $m$. Since the population size $\mu$ needs to be larger than $(2n/m+1)^{m/2}$ to preserve the non-dominated set, the runing time ultimately rises with the number $m$ of objectives.

\subsection{On the $m$OneJumpZeroJump Problem}

We prove in Theorem~\ref{thm:m-OJZJ} that the expected number of fitness evaluations for any MOEA solving $m$OneJumpZeroJump is $O(\mu n^k \cdot \min\{mn$,  $3^{m/2}\})$. By Definition~\ref{def:OJZJ}, if a solution $\bmx$ is a Pareto optimal, then for any block $B_i$, $|\bmx_{B_i}|_1 \in \{0,n'\} \cup [k..n'-k]$, where $n'=2n/m$ denotes the size of each block. We call $\bmx$ an \emph{internal} Pareto optimum if for any block $B_i$, $|\bmx_{B_i}|_1 \in [k..n'-k]$, and we call $\bmx$ an \emph{extreme} Pareto optimum if there exists a block $B_i$ such that $|\bmx_{B_i}|_1 \in \{0,n'\}$.
The proof idea is to divide the optimization procedure into three phases, where the first phase aims at finding an internal Pareto front point, the second phase finishes after finding the whole internal Pareto front whose analysis is similar to that of finding the Pareto front of $m$OneMinMax, and the third phase focuses on finding the extreme Pareto front points from the edge of the internal Pareto front.

\begin{mytheorem}\label{thm:m-OJZJ}
    For any MOEA solving $m$OneJumpZeroJump$_{n,k}$, if the algorithm preserves the non-dominated set with a maximum population size of $\mu$, employs uniform selection to select parent solutions, and employs bit-wise mutation to generate offspring solutions, the expected number of fitness evaluations for finding the Pareto front is $O(\mu n^k \cdot \min\{mn$,  $3^{m/2}\})$.
\end{mytheorem}

\begin{proof}
    We divide the optimization procedure into three phases. The first phase starts after initialization and finishes until an internal point in the Pareto front is found; the second phase starts after the first phase and finishes until all internal points in the Pareto front are found; the third phase starts after the second phase and finishes when all extreme points in the Pareto front are found.

   \textbf{The first phase} is to find an internal Pareto front point. If the initial population contains an internal Pareto optimum, the first phase is completed. We consider that the initial population does not contain an internal Pareto optimum. Then, for any solution $ \bm {x} $ in the population $P$, there exists a block $ B_i = [(i-1)n' + 1 .. in'], i \in [1..m/2] $ such that $ | \bm {x}_ {B_i} |_1 \in [0..k-1] \cup [n'-k + 1..n'] $. Without loss of generality, assume that $ | \bm {x}_ {B_i} |_1 \in [0..k-1] $. Then the probability of generating a solution $ \bm {y} $ with $ | \bm {y}_{B_i} |_1 \in [k,n'-k] $ from $ \bm {x} $ is at least
\begin{equation}\label{eq:OJZJ1}
\begin{aligned}
&\binom{n'-| \bm {x}_ {B_i} |_1}{k-| \bm {x}_ {B_i} |_1} \Big(1-\frac{1}{n} \Big)^{n-(k-| \bm {x}_ {B_i} |_1)} \Big( \frac{1}{n} \Big)^{k-| \bm {x}_ {B_i} |_1} \\
&\ge \frac{1}{e} \left(\frac{n'-| \bm {x}_ {B_i} |_1}{n(k-| \bm {x}_ {B_i} |_1)}\right)^{k-| \bm {x}_ {B_i} |_1}
 \ge \frac{1}{e}  \Big( \frac{n'-k}{nk} \Big)^k \ge \frac{1}{e}\Big(\frac{1}{mk}\Big)^k ,
\end{aligned}
\end{equation}

\noindent where the last inequality holds by $k\le n'/2$ (otherwise, internal Pareto optima will not exist).  The same bound also applies to $ | \bm {x}_{B_i} |_1 \in [n'-k+1 ..n']$. Let $J(\bmx)$ denote the number of blocks $B_{i}$ such that $|\bmx_{B_i}|_1 \in [k..n'-k]$, and let $J_{\max} := \max_{\bmx \in P} J(\bmx)$. For a solution $\bmx^*$ with $J(\bmx^*) = J_{\max}$, it will not be dominated by any solution $\bmy$ with $J(\bmy) < J_{\max}$. Thus, $J_{\max}$ will not decrease. When $J_{\max}$ reaches $m/2$, an internal Pareto front point is found. The probability of selecting the solution $\bmx^*$ corresponding to $J_{\max}$ is $1/\mu$, and the probability of increasing $J_{\max}$ by 1 through bit-wise mutation on $\bmx^*$ is at least $\frac{1}{e}(\frac{1}{m k})^k$ by Eq.~\eqref{eq:OJZJ1}. Therefore, in each generation, the probability of increasing $J_{\max}$ by $1$ is at least:
\begin{equation}\label{eq:OJZJ2}
\begin{aligned}
 1 - \left(1- \frac{1}{e\mu}\Big(\frac{1}{mk}\Big)^k\right)^{c\mu} \ge 1 - e^{-\frac{c}{e}(\frac{1}{mk})^k} \ge \frac{c}{c + e(mk)^k} \,\text{,}
 \\
\end{aligned}
\end{equation}
where the inequalities hold by $1+a \le e^{a}$ for any $a\in \mathbb{R}$. In each generation, the population produces $c\mu$ offspring solutions, thus the expected number of fitness evaluations of the first phase is at most $  (m/2)\cdot c \mu  \cdot (c+e(mk)^k)/c = O(\mu m^{k+1}k^k)$.

    \textbf{The second phase} is to find all internal Pareto front points. The analysis of the second phase is similar to that of $m$OneMinMax. After finding an internal Pareto front point, we show that the probability of finding all internal Pareto front points in $O(\mu n^2)$ fitness evaluations is at least $1 - o(1)$. For any internal Pareto front point $\bmv$, let $d_{\bmv} = \min_{\bmx \in F^I} \left\| \bmf(\bmx) - \bmv \right\|_{1}/2$ denote the minimum number of bit flips required to obtain $\bmv$ for all internal Pareto front points found so far, the set of which is denoted as $F^I$. Since all internal Pareto front points will be preserved, $d_{\bmv}$ cannot increase. By Eq.~\eqref{eq:OMM1}, the probability that $d_{\bmv}$ decreases is at least $cd_{\bmv}/(c+en)$. Let the random variable $X_i$ for $i\in[1..(m/2)(n'-2k)]$ denote the number of generations with $d_{\bmv} = i$. Let $X :=\sum_{i=1}^{m(n'-2k)/2}X_i$. Then similar to Eq.~\eqref{eq:OMM3}, we can derive that $\pr(X\ge 5en^2/c) \le e^{-n}$. In each generation, the population produces $c\mu$ offspring, thus the probability that the population contains no solution $\bmx$ with $\bmf(\bmx) = \bmv$ after $c\mu \cdot 5en^2 /c = O(\mu n^2)$ fitness evaluations is at most $e^{-n}$. 
    For any internal Pareto optimum $\bmx$, $f_{2i-1}(\bmx) + f_{2i}(\bmx) = 2k + n'$, and $ f_{2i-1}(\bmx), f_{2i}(\bmx) \in [2k, n']$ for $i\in[1..m/2]$.  Hence, the size of the internal Pareto front is $(n'-2k+1)^{m/2}$. By union bound, the probability of finding the whole internal Pareto front within $O(\mu n^2)$ fitness evaluations is at least $1-(n'-2k+1)^{m/2} \cdot e^{-n} = 1-o(1)$.
    
    \textbf{The third phase} is to find all extreme Pareto front points. Note that for any Pareto optimum $\bmx$, if there exists a block $B_i$ such that $|\bmx_{B_i}|_1 \in \{0, n'\}$, then we call $\bmx$ an extreme Pareto optimum and $\bmx_{B_i}$ an extreme block. Now, we use two bounds to show the running time of the third phase.

    First, we show that the expected number of fitness evaluations for finding all the extreme points in the Pareto front is $O(3^{m /2} \cdot \mu n^k )$. We divide the optimization procedure of the third phase into $m/2$ levels, and each solution at the $i$-th level-contains $i$ extreme blocks. We first consider the running time of finding all the extreme Pareto front points at level-$1$. After the second phase, all internal Pareto front points have been found. Hence, for any block $B_i$, there exist at least $\sigma_1 := (n'-2k+1)^{m/2-1}$ solutions $\bmx$ such that $|\bmx_{B_i}|_1 = k$. The probability of selecting such a solution as parent is at least $  \sigma_1/ \mu$, and the probability of flipping $k$ $1$-bits to obtain $\bmy$ such that $|\bmy_{B_i}|_1= 0$ is $(1 /n^k)\cdot (1-1 /n)^{n-k} \ge 1 / (en^k)$. The same result also applies for $|\bmx_{B_i}|_1 = n'-k$ to generate a solution $\bmy$ such that $|\bmy_{B_i}|_1= n'$. Therefore, in each generation, the probability of generating an extreme block is at least
    \begin{equation}\label{eq:OJZJ3}
    \begin{aligned}
     1 - \Big(1- \frac{\sigma_1}{e\mu n^k} \Big)^{c\mu} \ge 1 - e^{-\frac{c\sigma_1}{en^k}} \ge \frac{c\sigma_1}{c\sigma_1+en^k},
    \end{aligned}
    \end{equation}
    where the inequalities hold by $1+a \le e^{a}$ for any $a\in \mathbb{R}$. Thus, the expected number of fitness evaluations for generating an extreme block is at most $c\mu\cdot (c\sigma_1+ e n^k)/(c\sigma_1) = O(\mu n^k/\sigma_1)$. Then, we consider finding all the solutions at level-$1$ where $B_i$ is an extreme block. 
    This optimization process is similar to the second phase, extending from an internal Pareto front point to the whole internal Pareto front, and the difference is no need to consider the block $B_i$. Therefore, the expected number of fitness evaluations of this process is less than the second phase, which is $O(\mu n^2)$. Because there can be $m/2$ positions with extreme blocks, and each extreme block has $2$ possible forms $0^{n'}$ and $1^{n'}$, the expected number of fitness evaluations to find all level-$1$ solutions is $O((m/2) \cdot 2 \cdot \mu n^k/\sigma_1) + O((m/2) \cdot 2 \cdot \mu n^2) = O(m \mu (n^k/\sigma_1+n^2))$. Next, we consider finding all solutions at level-$2$. The number of combinations of $2$ extreme blocks from $m/2$ positions is $\binom{m /2}{2} \cdot 2^2$, and each combination can be obtained by adding an extreme block to the solution at level-$1$. Assume that there is no solution $\bmx$ in the population such that both $B_i$ and $B_j$ (where $i \neq j$) are extreme blocks. 
    After finding all solutions at level-$1$, there exists at least $\sigma_2 := (n'-2k+1)^{m/2-2}$ solutions that can 
    be used to generate a solution at level-$2$ (with $B_i$ and $B_j$ as the extreme blocks)  by mutating $k$ bits. Then, following the above analysis, the expected number of fitness evaluations for generating a solution at level-$2$ is at most $O(\mu n^k/\sigma_2)$.
    The expected number of fitness evaluations to extend from a solution at level-$2$ to all solutions at level-$2$ is also less than that of the second phase, which is at most $O(\mu n^2)$. Thus, the expected number of fitness evaluations to find all level-$2$ solutions is $O(\binom{m /2}{2}\cdot 2^2 \cdot \mu (n^k/\sigma_2 + n^2))$. Continuing this way, for each level-$i$, we have $\sigma_i := (n'-2k+1)^{m/2-i}$ and the expected number of fitness evaluations to find all Pareto front points from level-$1$ to level-$m/2$ is at most 
    $O(\sum_{i = 1}^{m/2} \binom{m /2}{i}\cdot 2^i \mu (n^k/\sigma_i + n^2)) = O((2 + 1/(n'-2k+1))^{m/2} \mu n^k + 3^{m/2} \mu n^2)$. Since $k \le n'/2$, the inequality $1/(n'-2k+1) \le 1$ holds, implying that the expected number of fitness evaluations for finding all the extreme Pareto front points is at most $O(3^{m/2} \mu n^k)$.

    Then we show that after the second phase, all the extreme Pareto front points can be found in $O(m \mu n^{k+1})$ number of fitness evaluations with probability $1 - e^{-\Omega(mn)}$. 
    After the second phase, all internal Pareto front points have been found. Assume that a Pareto front point $\bmv$ containing $d$ extreme values (the objective value is $k$ or $n'+k$) has not been found, then there exists a solution $\bmx \in P$ such that except for the $d$ extreme values of $\bmv$, the remaining objective values are all equal to $\bmv$. Then, $\bmv$ could be extended from $\bmx$ by generating $d/2$ extreme blocks. 
    Let the random variable $X_i, 1\le i\le m/2$, denote the number of generations to generate the $i$-th extreme block. Let $X := \sum_{i=1}^{m/2} X_i$.
    By the analysis of extreme blocks, the probability of generating an extreme block in each generation is at least $c/(c+en^k)$. Let $p_i:= c/(c+en^k)$. Hence, $X_1,\ldots,X_{m/2}$ stochastically dominates independent geometric random variables $Y_1,\ldots,Y_{m/2}$, where the success probability of $Y_i$ is $p_i$. Moreover, for $Y:= \sum_{i=1}^{m/2} Y_i$, $X$ stochastically dominates $Y$. By Theorem 1 in~\cite{witt2014fitness}, we have
    \begin{equation}\label{eq:OJZJ4}
    \begin{aligned}
    &\pr\Big(X \ge \frac{mn^{k+1}}{c}\Big) \le \pr\Big(Y \ge \frac{mn^{k+1}}{c} \Big) = \pr( Y \ge E[Y] + \delta) \\ 
    &\le e^{-\frac{1}{4}\min\{\frac{\delta^2}{s},\frac{c \delta}{c + en^k}\}} \le e^{-\Omega(mn)},
    \end{aligned}
    \end{equation}
    where $\delta = mn^{k+1}/c - m(c + en^k)/(2c)$, $E[Y] = m(c + en^k)/(2c) $ and $s = m(c+en^k)^2 / (2c^2) $. Thus, the probability that the population contains no solution $\bmx$ with $\bmf(\bmx) = \bmv$ in $c \mu \cdot (mn^{k+1}/c) = m \mu n^{k+1} $ number of fitness evaluations is at most $e^{-\Omega(mn)}$. For Pareto optimal solutions, each block corresponds to two objective values, and the value range is $\{k, n'+k \} \cup [2k..n']$. Thus, the size of the extreme Pareto front is $(n'-2k +3)^{m/2} - (n'-2k+1)^{m/2}$. By applying the union bound, the probability of finding all the extreme Pareto front points within $m\mu n^{k+1}$ number of fitness evaluations is $1-((n'-2k +3)^{m/2} - (n'-2k+1)^{m/2})\cdot e^{-\Omega(mn)} \ge 1 - n^{m/2} e^{-\Omega(mn)} = 1 - o(1)$. Let each $m \mu n^{k+1}$ fitness evaluations to find the Pareto front be an independent trial with success probability of $1-o(1)$. Thus, the expected number of fitness evaluations to find the Pareto front is at most  $(1-o(1))^{-1} m \mu n^{k+1} = O(m \mu n^{k+1})$. Finally, combining the above analysis, the third phase needs $O(\mu n^k \cdot \min\{mn$,  $3^{m/2}\})$ number of fitness evaluations in expectation.

    Combining the three phases, the upper bound on the expected number of fitness evaluations for finding the whole Pareto front is $O(\mu m^{k+1} k^k + \mu n^2 + \mu n^k \cdot \min\{mn$,  $3^{m/2}\}) = O(\mu n^k \cdot \min\{mn$,  $3^{m/2}\})$.
    \qed
\end{proof}

Theorem~\ref{thm:m-OJZJ} shows that for any MOEA preserving the non-dominated set, the number of fitness evaluations for solving $m$OneJumpZeroJump is $O(\mu n^k \cdot \min\{mn$,  $3^{m/2}\})$ in expectation. Note that when $m$ is a constant, the expected number of fitness evaluations is $O(\mu n^k)$. Conversely, if the number $m$ of objectives is large, e.g., $m=n/4$, the expected number of fitness evaluations is $O(\mu n^{k+2})$.

\section{Application to Running Time Analysis of SPEA2} \label{sec:analysis_spea2}

In this section, we first show that when the archive size of SPEA2 is large enough, SPEA2 can preserve the non-dominated set. Then, we apply the results in the previous section, i.e., Theorems~\ref{thm:m-OMM}, \ref{thm:m-LOTZ} and~\ref{thm:m-OJZJ}, to derive the running time of SPEA2.

\subsection{ Large Archive Preserves Non-dominated Solutions}

Throughout the process of SPEA2, the next population $P$ consists of the mutated individuals from the archive $A$. That is, the archive $A$ in SPEA2 actually plays the role of population, while the population $P$ acts as offspring population. We now show that when the archive size $\bar{\mu}$ is always no less than the maximum cardinality of the set of non-dominated solutions having different objective vectors for any $m$-objective function $\bmf$, SPEA2 will preserve the non-dominated set on $\bmf$.

\begin{mylemma}\label{lemma:SPEA2}
        For SPEA2 solving $\bmf$, let $A'$ denote the non-dominated set of solutions in the union of the population $P$ and archive $A$. If the size $\bar{\mu}$ of archive is always no less than $|\bmf(A')|$, then SPEA2 preserves the non-dominated set.
\end{mylemma}

\begin{proof}
    Consider that all the non-dominated solutions in $P\cup A$ have been added to $A'$ (line~4 of Algorithm~\ref{alg:SPEA2}). If the size of $A'$ is at most $\bar{\mu}$, then all the non-dominated solutions will survive to the next generation. Thus, the lemma holds. If the size of $A'$ is larger than $\bar{\mu}$ (i.e., $|A'| > \bar{\mu}$), SPEA2 will use the truncation operator to remove the redundant solutions one by one, until $|A'| = \bar{\mu}$. Let $D_{\bmz} = \{\bmy \in A' \mid \bmf(\bmz) = \bmf(\bmy)\}$. Since $\bar{\mu}$ is at least $|\bmf(A')|$ and $|A'| > \bar{\mu}$, there exists $\bmz \in A'$ such that $|D_{\bmz}| > 1$. Note that $\sigma_{\bmx}^1 > 0$ if $|D_{\bmx}| = 1$, and $\sigma_{\bmx}^1 =0$ if $|D_{\bmx}| > 1$. Hence, if $|D_{\bmx}| = 1$, then $\sigma_{\bmx}^1 > 0 = \sigma_{\bmz}^1$, implying that $\bmz \le_{d} \bmx$. Thus, $\bmz$ will be removed before $\bmx$, which implies that the non-duplicated solutions (regarding the objective vectors) will not be removed from $A'$. Therefore, the lemma holds.
    \qed
\end{proof}

By Lemma~\ref{lemma:SPEA2}, we can set the size of the archive $A$ accordingly, so that SPEA2 preserves the non-dominated set, which then allows us to apply Theorems~\ref{thm:m-OMM} to~\ref{thm:m-OJZJ} to derive the expected running time of SPEA2.

\subsection{ Running Time of SPEA2 on Benchmark Problems }

Let $S$ denote the set of non-dominated solutions in the union of the population $P$ and archive $A$. In the following, we will give the maximum values of $|\bmf(S)|$ when SPEA2 solves $m$OneMinMax, $m$LeadingOnesTrailingZeroes and $m$OneJumpZeroJump, res-pectively, so as to determine the size of the archive. Then, by applying Lemma~\ref{lemma:SPEA2} and the corresponding theorem, we provide the expected number of fitness evaluations for SPEA2 to find the Pareto front.

\begin{mytheorem}
    For SPEA2 solving $m$OneMinMax, if the size $\bar{\mu}$ of the archive is at least $(2n/m + 1)^{m/2}$, uniform selection is employed to select parent solutions, and bit-wise mutation or one-bit mutation is employed to generate offspring solutions, the expected number of fitness evaluations for finding the Pareto front is $O(\bar{\mu} n\cdot \min\{m\log n, n\})$.
\end{mytheorem}

\begin{proof}
    Because any solution is Pareto optimal, $|\bmf(S)|$ is always at most the size of the Pareto front. By Definition~\ref{def:OMM}, the size of the Pareto front is $(2n/m + 1)^{m/2}$, implying that $|\bmf(S)| \le (2n/m + 1)^{m/2}$. According to Lemma~\ref{lemma:SPEA2} and Theorem~\ref{thm:m-OMM}, if the size $\bar{\mu}$ of the archive in SPEA2 is at least $(2n/m + 1)^{m/2}$, the whole Pareto front can be found in $O(\bar{\mu} n\cdot \min\{m\log n, n\})$ expected number of fitness evaluations.
    \qed
\end{proof}

\begin{mytheorem}
    For SPEA2 solving $m$LeadingOnesTrailingZeroes, if the size $\bar{\mu}$ of the archive is at least $(2n/m + 1)^{m-1}$, uniform selection is employed to select parent solutions, and bit-wise mutation or one-bit mutation is employed to generate offspring solutions, the expected number of fitness evaluations for finding the Pareto front is $O(\bar{\mu} n^2)$.
\end{mytheorem}

\begin{proof}
    By Lemma 4.2 in \cite{nsga3}, the maximum cardinality of the set of non-dominated solutions having different objective vectors is at most $(2n/m +1)^{m-1}$ for $m$LeadingOnes-Trailingzeroes, implying $|\bmf(S)|\le (2n/m +1)^{m-1}$. According to Lemma~\ref{lemma:SPEA2} and Theorem~\ref{thm:m-LOTZ}, if the size $\bar{\mu}$ of the archive in SPEA2 is no less than $(2n/m +1 )^{m-1}$, the whole Pareto front can be found in $O(\bar{\mu} n^2)$ fitness evaluations in expectation.
    \qed
\end{proof}


\begin{mytheorem}
    For SPEA2 solving $m$OneJumpZeroJump, if the size $\bar{\mu}$ of the archive is at least $(2n/m - 2k +3 )^{m/2}$, uniform selection is employed to select parent solutions, and bit-wise mutation is employed to generate offspring solutions, the expected number of fitness evaluations for finding the Pareto front is $O(\bar{\mu} n^k \cdot \min\{mn$,  $3^{m/2}\})$.
\end{mytheorem}

\begin{proof}
    By Lemma 3 in \cite{zheng2024runtime}, if $k\le n'/2$, the maximum cardinality of the set of non-dominated solutions having different objective vectors is at most $(2n/m -2k +3)^{m/2}$ for $m$OneJump-ZeroJump, implying $|\bmf(S)| \le (2n/m - 2k + 3)^{m/2}$. According to Lemma~\ref{lemma:SPEA2} and Theorem~\ref{thm:m-OJZJ}, if the archive size $\bar{\mu} \geq (2n/m -2k+ 3)^{m/2}$, the whole Pareto front can be found in $O(\bar{\mu} n^k \cdot \min\{mn$,  $3^{m/2}\})$ fitness evaluations in expectation.
    \qed
\end{proof}

\section{Application to Other Algorithms}\label{sec:verification}

In this section, we show that Theorems \ref{thm:m-OMM}, \ref{thm:m-LOTZ} and \ref{thm:m-OJZJ} can be applied to other MOEAs (including SEMO, NSGA-II, SMS-EMOA, etc.) to derive their expected running time, which are consistent with previous results.

\textbf{$m$OneMinMax:} For the bi-objective OneMinMax problem (i.e., $m$OneMinMax with $m=2$), because the SEMO algorithm always maintains the non-dominated set, the expected number of fitness evaluations to find the entire Pareto front is $O(\mu n \log n)$ by Theorem~\ref{thm:m-OMM}, where the equality holds because the population size is at most the Pareto front size $n+1$. This aligns with Giel and Lehre's result~\cite{giel2006effect}. 
Zheng and Doerr~\cite{zheng2023first} proved that for a population size $\mu$ exceeding $4$ times the Pareto front size (i.e., $\mu \ge 4(n+1)$), NSGA-II preserves the non-dominated set on OneMinMax. Then, by Theroem~\ref{thm:m-OMM}, the expected number of fitness evaluations to find the Pareto front is $O(\mu n \log n) $, aligning with their result. Similarly, Bian \textit{et al.}~\cite{bian2023stochastic} and Nguyen \textit{et al.}~\cite{nguyen2015sibea} proved that when the population size is greater than the Pareto front size (i.e., $n+1$), both SMS-EMOA and $(\mu+1)$SIBEA preserve the non-dominated set on OneMinMax, resulting in an expected number of fitness evaluations of $O(\mu n \log n)$, consistent with Theorem \ref{thm:m-OMM}. For $m$OneMinMax with large constant $m$, Andre \textit{et al.}~\cite{nsga3} proved that NSGA-III preserves the non-dominated set on $m$OneMinMax, if it employs a set of reference points $\mathcal{R}_p$ with $p \ge 4n\sqrt{m}$, and with a population size of $\mu \ge (2n / m + 1)^{m / 2}$. Then, they derived that the expected number of fitness evaluations is $O(\mu n \log n)$, which is consistent with the result derived from Theorem \ref{thm:m-OMM}.

\textbf{$m$LeadingOnesTrailingZeroes:}  For the bi-objective LeadingOnesTrailingZeroes  problem (i.e., $m$LeadingOnesTrailingZeroes with $m=2$), Laumanns \textit{et al.}~\cite{LaumannsTEC04} and Giel~\cite{giel2003expected} proved that the expected number of fitness evaluations of SEMO and GSEMO for finding the Pareto front is $O(n^3)$. Brockhoff~\cite{brockhoff2008analyzing} proved that when the population size is greater than the Pareto front size (i.e., $n+1$), $(\mu+1)$SIBEA preserves the non-dominated set on LeadingOnesTrailingZeroes, then the Pareto front can be found in $O(\mu n^2)$ number of fitness evaluations in expectation. Zheng and Doerr~\cite{zheng2023first} proved that for the population size at least $4(n+1)$, NSGA-II preserves the non-dominated set on LeadingOnesTrailingZeroes and the expected number of fitness evaluations is also $O(\mu n^2)$. These results are all consistent with Theorem \ref{thm:m-LOTZ}. For $m\ge 4$, Laumanns \textit{et al.} \cite{LaumannsTEC04} proved that the expected running time of SEMO is $O(n^{m+1})$, which is consistent with the result in Theorem \ref{thm:m-LOTZ}, because the maximal population size during the optimization procedure is at most $(2n / m+1)^{m-1} \le n^{m-1}$. Andre \textit{et al.}~\cite{nsga3} proved that NSGA-III preserves the non-dominated set on $m$LeadingOnesTrailingZeroes, if it employs a set of reference points $\mathcal{R}_p$ with $p \ge 4n\sqrt{m}$, and with a population size of $\mu \ge (2n / m + 1)^{m-1}$. Then, they derived that the expected number of fitness evaluations is $O(\mu n^2)$, which is consistent with Theorem \ref{thm:m-LOTZ}.

\textbf{$m$OneJumpZeroJump:} For the bi-objective OneJumpZeroJump problem (i.e.,  $m$-OneJumpZeroJump with $m=2$), Doerr \textit{et al.}~\cite{doerr2021ojzj} proved that the expected number of fitness evaluations for GSEMO finding the Pareto front is $O((n-2k)n^k)$, which is consistent with Theorem~\ref{thm:m-OJZJ}, because the maximal population size during the optimization procedure is at most $(n-2k+3)$. When the population size is at least $4(n-2k+3)$ for NSGA-II~\cite{doerr2023first} and $(n-2k+3)$ for SMS-EMOA~\cite{bian2023stochastic}, both algorithms preserve the non-dominated set on OneJumpZeroJump and can find the Pareto front within $O(\mu n^k)$ number of fitness evaluations in expectation, also consistent with Theorem~\ref{thm:m-OJZJ}. For $m \le \log n$, Zheng and Doerr~\cite{zheng2024runtime} proved that when the population size is at least $(2n/m-2k+3)^{m/2}$, SMS-EMOA preserves the non-dominated set and can find the Pareto front within $O(M\mu n^k)$ number of fitness evaluations in expectation, where $M = (2n/m-2k+3)^{m/2}$, while Theorem~\ref{thm:m-OJZJ} shows that any MOEA that can preserve the non-dominated set can find the Pareto front in $O(3^{m/2} \mu n^k)$ expected number of fitness evaluations, implying a tighter bound.

\section{Conclusion and Discussion} \label{conclusion}


In this paper, we give general theorems for analyzing the expected running time of
MOEAs on the three common benchmark problems, $m$OneMinMax, $m$LeadingOnes-Trailingzeroes and $m$OneJumpZeroJump. These theorems show that if the population of MOEAs preserves the objective vectors of non-dominated solutions, their running time can be upper bounded. We apply them to derive the expected running time of SPEA2 for the first time, and also to analyze other MOEAs, deriving results consistent with previously known ones. We hope these theorems can be useful for theoretical analysis of MOEAs in the future. The theorems also suggest that when the population size is large enough, the analytical behavior of MOEAs may tend to be similar. Thus, it would be interesting to analyze the approximation performance of MOEAs with small populations, to better understand their differences.

In a parallel theoretical work, Wietheger and Doerr~\cite{wietheger2024near} proved near-tight running time guarantees for GSEMO on $m$OneMinMax, $m$CountingOnesCountingzeroes, $m$LeadingOnesTrailingZeroes, and $m$OneJumpZeroJump, and also transferred similar bounds to SEMO, SMS-EMOA and NSGA-III. Note that their proof techniques are superior and the bounds are tighter, as they not only considered the probability \( p \) of not finding a Pareto front point within a given number \( T \) of iterations, but also refined this process by considering each block individually. Specifically, they considered the probability \( p' \) of not achieving the optimization goal for each block within the given time. Using the union bound leads to \( p \leq m' p' \), where \( m' \) denotes the number of blocks. Let \( M \) represent the size of the Pareto front. Then, the probability of finding the complete Pareto front within \( T \) iterations is at least \( 1 - Mp \ge 1 - Mm'p' \). Their analysis of the optimization of blocks using the union bound allows them to obtain tighter time bounds.

\section*{Acknowledgments}
The authors want to thank the anonymous reviewers for their helpful comments and suggestions. This work was supported by the National Science and Technology Major Project (2022ZD0116600) and National Science Foundation of China (62276124). Chao Qian is the corresponding author. The conference version of this paper has appeared at PPSN’24.

%
%
%
\bibliographystyle{splncs04}
\bibliography{ppsn_spea2}
%


\clearpage



\end{document}